\RequirePackage[l2tabu, orthodox]{nag} %

\documentclass[onefignum,onetabnum]{siamonline190516}

\usepackage{amsmath} %
\usepackage{amssymb} %
\usepackage{amsfonts} %
\usepackage{amsopn} %

\usepackage{mathrsfs} %

\newsiamthm{claim}{Claim}
\newsiamremark{remark}{Remark}
\newsiamremark{hypothesis}{Hypothesis}
\crefname{hypothesis}{Hypothesis}{Hypotheses} %
\newsiamremark{assumption}{Assumption}
\crefname{assumption}{Assumption}{Assumptions} %
\usepackage{enumitem}
\setlist[enumerate]{leftmargin=.5in}
\setlist[itemize]{leftmargin=.5in}

\usepackage{algorithm} %
\usepackage{algpseudocode} %
\usepackage{algpascal}     %
\usepackage{algc}          %

\usepackage{graphicx} %
\usepackage{epstopdf}
\ifpdf
  \DeclareGraphicsExtensions{.eps,.pdf,.png,.jpg}
\else
  \DeclareGraphicsExtensions{.eps}
\fi

\usepackage{microtype} %
\usepackage[strict=true]{csquotes} %
\usepackage[strings]{underscore} %

\headers{Normal-bundle Bootstrap}{R. Zhang and R. Ghanem}

\title{Normal-bundle Bootstrap\thanks{Submitted to the editors.
\funding{This work was supported, in part, by the National Science Foundation grant DMS-1638521.}}}

\author{Ruda Zhang\thanks{The Statistical and Applied Mathematical Sciences Institute, Durham, NC 
    (\email{rzhang@samsi.info}).
    Department of Mathematics, North Carolina State University, Raleigh, NC
    (\email{rzhang27@ncsu.edu}).}
  \and Roger Ghanem\thanks{Department of Civil and Environmental Engineering,
    University of Southern California, Los Angeles, CA 
    (\email{ghanem@usc.edu}).}}

\hypersetup{
  pdftitle={Normal-bundle Bootstrap},
  pdfauthor={Ruda Zhang and Roger Ghanem}
  pdfprintscaling=None,
}

\begin{document}

\maketitle

\begin{abstract} %
  Probabilistic models of data sets often exhibit salient geometric structure.
  Such a phenomenon is summed up in the manifold distribution hypothesis,
  and can be exploited in probabilistic learning.
  Here we present normal-bundle bootstrap (NBB), a method that generates new data
  which preserve the geometric structure of a given data set.
  Inspired by algorithms for manifold learning and concepts in differential geometry,
  our method decomposes the underlying probability measure into
  a marginalized measure on a learned data manifold and conditional measures on the normal spaces.
  The algorithm estimates the data manifold as a density ridge,
  and constructs new data by bootstrapping projection vectors and adding them to the ridge.
  We apply our method to the inference of density ridge and related statistics,
  and data augmentation to reduce overfitting.
\end{abstract}



\section{Introduction}
\label{sec:intro}

When data sets are modeled as multivariate probability distributions,
such distributions often have salient geometric structure.
In regression, the joint probability distribution of explanatory and response variables
is centered around the response surface.
In representation learning and deep learning,
a common assumption is the manifold distribution hypothesis,
that natural high-dimensional data concentrate close to a nonlinear low-dimensional manifold
\cite{Bengio2013, Goodfellow2016}.
In topological data analysis, including manifold learning,
the goal is to capture such structures in data
and exploit them in further analysis \cite{Wasserman2018}.

The goal of this paper is to present a method that generates new data,
which preserve the geometric structure of a probability distribution modeling the given data set.
As a variant of the bootstrap resampling method,
it is useful for the inference of statistical estimators.
Our method is also useful for data augmentation,
where one wants to increase training data diversity to reduce overfitting,
without collecting new data.

Our method is inspired by constructions in differential geometry
and algorithms for nonlinear dimensionality reduction.
Principal component analysis of a data set decomposes the Euclidean space of variables
into orthogonal subspaces, in decreasing order of maximal data variance.
If we consider the first few principal components to represent
the geometry of the underlying distribution,
and the remaining components to represent the normal space to the principal component space,
we decompose the distribution into one on the principal component space
and noises in the normal spaces at each point of the principal component space.
Normal bundle of a manifold embedded in a Euclidean space generalizes such linear decomposition,
such that every point in a neighborhood of the manifold can be uniquely represented as
the sum of its projection on the manifold and the projection vector.
There are a few concepts that generalize principal components to nonlinear summaries of data.
Principal curve \cite{Hastie1989} and, more generally, principal manifold
is a smooth submanifold where each point is the expectation of the distribution in its normal space.
More recently, \cite{Ozertem2011} proposed a variant called density ridge,
where each point is the mode of the distribution in a neighborhood in its normal space.
Density ridge is locally defined and is estimated by subspace-constrained mean shift (SCMS),
a gradient descent algorithm.
Compared with principal curve algorithms, the SCMS algorithm is much faster,
applicable to any manifold dimension, robust to outliers, and the ridge is fully learned from data.

Normal-bundle bootstrap (NBB) picks a point on the estimated density ridge and
adds to it the projection vector of a random point,
whose projection is in a neighborhood of the picked point on the ridge.
With this procedure, the distribution on the ridge is preserved,
while distributions in the normal spaces are locally randomized.
Thus, the generated data will have greater diversity and remain consistent
with the original distribution, including its geometric structure.
Our method should work well for data sets in any Euclidean or Hilbert space,
as long as the underlying distribution is concentrated around a low-dimensional submanifold,
and the sample size is sufficient for the manifold dimension.
\Cref{fig:overview}a-c illustrates density ridge, its normal bundle,
and the normal-bundle bootstrap algorithm.

\subsection{Related literature}
\label{sub:review}

Within bootstrap methods,
normal-bundle bootstrap is mostly close to residual bootstrap in regression analysis,
but our method is in the context of dimension reduction.
Residual bootstrap fits a regression model on the data,
and adds random residual in the response variables to each point on the fitted model,
assuming the errors are identically distributed.
Such residuals in our context are the projection vectors.
Because the normal spaces on a manifold are not all parallel in general,
we cannot bootstrap all the projection vectors.
Instead, we only assume that the distributions in the normal spaces are continuously varying
over the density ridge, and bootstrap nearby projection vectors.
Also in regression analysis, wild bootstrap \cite{WuCF1986} allows for heteroscedastic errors,
and bootstraps by flipping the sign of each residual at random,
assuming error distributions are symmetric.
Such assumption does not apply in our context,
because each point on the density ridge is the mode of the distribution on a normal disk,
which can be asymmetric and biased in general.

In probabilistic learning on manifolds,
\cite{Soize2016} proposed a Markov chain Monte Carlo (MCMC) sampler to generate new data sets,
which preserve the concentration of probability measure estimated from the original data set
\cite{Soize2020b} and have applications in uncertainty quantification \cite{ZhangRD2020EnvEcon}.
This paper handles the same problem, but explicitly estimates the manifold by the density ridge,
and generates new data by bootstrapping, which avoids the computational cost of MCMC sampling.

There is a large literature at the broad intersection of differential geometry and statistics.
For parametric statistics on special manifolds with analytic expression,
which includes directional statistics, see \cite{Chikuse2003}.
For nonparametric statistical theory on manifolds and its applications,
especially for shape and object data, see \cite{Bhattacharya2012} and \cite{Patrangenaru2015}.
Statistical problems on submanifolds defined by implicit functions
are studied recently in \cite{ChenYC2020}.

Several MCMC methods have been proposed to sample from probability distributions
on Riemannian submanifolds:
\cite{Brubaker2012} proposed a general constrained framework of Hamiltonian Monte Carlo (HMC)
methods for manifolds defined by implicit constraints;
\cite{Byrne2013} proposed a similar HMC method,
but for manifolds with explicit forms of tangent spaces and geodesics.

The machine learning and deep learning communities also have various methods for
estimating and sampling from probability densities with salient geometric structures.
Manifold Parzen windows (MParzen) algorithm \cite{Vincent2002}
is a kernel density estimation method which captures the data manifold structure.
The estimated density function is easy to sample from, and we compare it with our method. %
Denoising auto-encoder \cite{Bengio2013b} is a feed-forward neural network
that implicitly estimates the data-generating distribution,
and can sample from the learned model by running a Markov chain
that adds noise and samples from the learned denoised distribution iteratively.
Normalizing flow \cite{Papamakarios2019} is a deep neural network
that represents a parametric family of probabilistic models,
which is the outcome of a simple distribution mapped through
a sequence of simple, invertible, differentiable transformations.
It can be used for density estimation, sampling, simulation, and parameter estimation.

\begin{figure}[t]
  \centering
  \includegraphics[width=\linewidth]{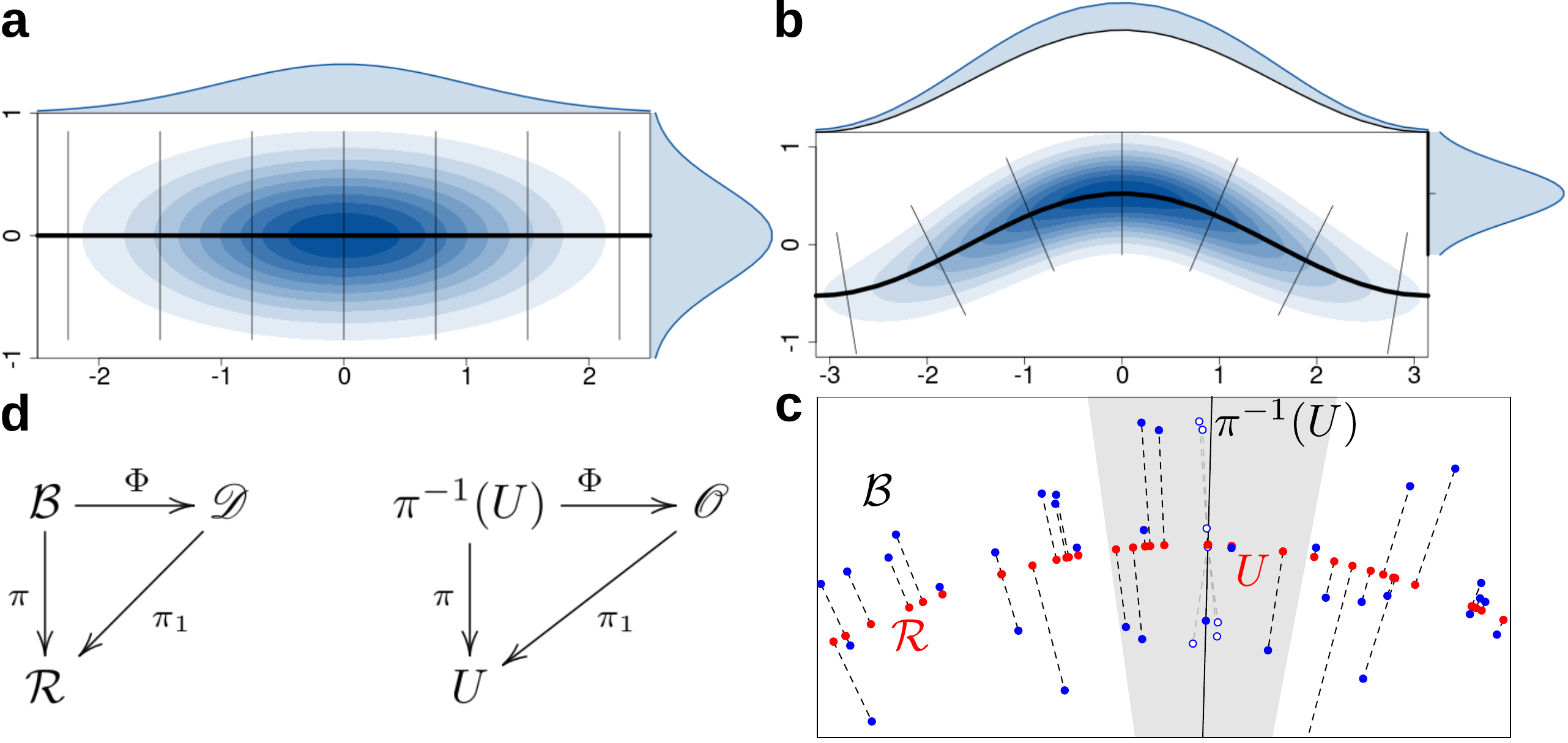}
  \caption{Density ridge, normal bundle, and NBB.
    (a) For a 2d Gaussian PDF (blue contours),
    its 1d density ridge (bold line) is its 1st principal component line,
    where the normal spaces (thin lines) are parallel to the 2nd principal component.
    Probability density on the normal spaces (right margin)
    declines faster than that on the ridge (top margin).
    (b) In general, density ridges are nonlinear, and its normal bundle decomposes
    the original distribution into one on the ridge and one on each normal space.
    (c) The NBB algorithm moves data points (solid blue) to the ridge (red)
    and for each point on the ridge,
    picks neighboring points on the ridge (shaded segment) and
    adds the projection vectors (dashed line) to construct new data points (hollow blue).
    Using a smooth frame can keep the constructed points in the normal space.
    (d) Commutative diagrams of normal bundle:
    $\mathcal{B}$, basin of attraction; $\mathcal{R}$, density ridge;
    $\pi$, projection; $U$, a neighborhood in density ridge.}
  \label{fig:overview}
\end{figure}

\section{Mathematics:
  geometric decomposition of Euclidean spaces and probability measures}
\label{sec:math}

Consider a probability measure $\mu$ on the Euclidean space $\mathbb{R}^n$,
which has a probability density function (PDF) $p$.
Given a data set $X$ which is a random sample of size $N$ from $\mu$,
we want to generate new data that are distinct from $X$, but consistent with $\mu$.
In particular, we want to solve this problem more efficiently
by exploiting the geometric structure of $p$,
which may be represented by a submanifold $\mathcal{R}$ of $\mathbb{R}^n$.
The mathematical foundation of our method is to decompose $\mathbb{R}^n$
into a collection $\{\mathcal{F}_r\}_{r \in \mathcal{R}}$ of submanifolds
indexed by points in $\mathcal{R}$,
where each submanifold $\mathcal{F}_r$ intersects $\mathcal{R}$ at $r$ orthogonally.
In this way, $\mu$ also gets decomposed into probability measures on submanifolds
$\mathcal{R}$ and each $\mathcal{F}_r$.

\begin{definition}
  \label{def:ridge}
  Ridge of dimension $d \in \{0, \dots, n\}$ for a twice differentiable function
  $f: \mathbb{R}^n \mapsto \mathbb{R}$, denoted as $\text{Ridge}(f, d)$,
  is the set of points where the $c = n - d$ smallest eigenvalues of the Hessian are negative,
  and the span of their eigenspaces are orthogonal to the gradient:
  $\text{Ridge}(f, d) := \{x \in \mathbb{R}^n : \lambda_c < 0, L g = 0\}$.
  Here, Hessian $H = \nabla \nabla f$ has an eigen-decomposition $H = V \Lambda V^{\text{T}}$,
  where $\Lambda = \text{diag}(\lambda)$ and $\lambda = (\lambda_i)_{i=1}^n$ is in increasing order.
  Let $V = (V_c; V_d)$ where $V_c$ and $V_d$ are column matrices
  of $c$ and $d$ eigenvectors respectively.
  Denote projection matrices $U = V_d V_d^{\text{T}}$, $L = V_c V_c^{\text{T}} = I - U$,
  and gradient $g = \nabla f$.
\end{definition}

\begin{assumption}
  \label{ass:basic}
  Let $D = \{x \in \mathbb{R}^n : p(x) > 0\}$, assume that:
  (1) $p|_D \in C^2(D, \mathbb{R}_{> 0})$;
  (2) for some $d \in \{1, \dots, n-1\}$,
  $\text{Ridge}(p, d) \subset D$ is an embedded $d$-dimensional submanifold of $\mathbb{R}^n$.
\end{assumption}

\textit{Density ridge} \cite{Ozertem2011} is a ridge of a probability density function.
With \cref{ass:basic},
(1) guarantees that $\text{Ridge}(p, d)$ is well-defined for every $d \in \{0, \dots, n\}$;
per the manifold distribution hypothesis we also require (2), and with the specific $d$
we define $\mathcal{R} = \text{Ridge}(p, d)$ and codimension $c = n - d$.
We note that this manifold assumption on density ridge is not very restrictive.
In fact, it is analogous to a modal regression problem
that assumes the conditional modes not to bifurcate.
The remaining part of this section lays out the related mathematical concepts
in differential geometry, measure theory, and dynamical system.

\subsection{Differential geometry}
We call the Euclidean space of dimension $n$ the Euclidean $n$-space;
similarly, if a manifold has dimension $d$, we call it a $d$-manifold.
An \textit{embedded submanifold} $(\mathcal{M}, \mathcal{T}, \mathcal{A})$ of $\mathbb{R}^n$
is a subset $\mathcal{M} \subset \mathbb{R}^n$
endowed with the subspace topology $\mathcal{T}$ and the subspace smooth structure $\mathcal{A}$,
such that the inclusion map $\iota: \mathcal{M} \mapsto \mathbb{R}^n$ is smooth
and its differential has full rank.
A \textit{Riemannian submanifold} $(\mathcal{M}, g)$ of $\mathbb{R}^n$ is an embedded submanifold
$\mathcal{M}$ endowed with the induced Riemannian metric $g = \iota^* \bar g$,
where $\bar g$ is the Euclidean metric (the standard Riemannian metric on $\mathbb{R}^n$)
and $\iota^*$ is the pullback operator by $\iota$.
In the following, $\mathcal{M}$ denotes a Riemannian $d$-submanifold of $\mathbb{R}^n$.
At a point $p \in \mathcal{M}$,
\textit{tangent space} $T_p \mathcal{M}$ is the $d$-dimensional vector space
consisting of all the vectors tangent to $\mathcal{M}$ at $p$,
and \textit{normal space} $N_p \mathcal{M}$
is the $c$-dimensional orthogonal complement to $T_p \mathcal{M}$.
The \textit{normal bundle} $N \mathcal{M}$ is the disjoint union of all the normal spaces:
$N \mathcal{M} = \sqcup_{p\in\mathcal{M}} N_p \mathcal{M}$.
It is often identified with the product manifold $\mathcal{M} \times \mathbb{R}^c$
so that its elements can be written as $(p, v)$,
where $p \in \mathcal{M}$, $v \in N_p \mathcal{M} \cong \mathbb{R}^c$.
The \textit{natural projection} of $N \mathcal{M}$ is the map
$\pi_1: N \mathcal{M} \mapsto \mathcal{M}$ such that $\pi_1(p, v) = p$.

We focus on neighborhoods of $\mathcal{M}$ in $\mathbb{R}^n$
that are diffeomorphic images of open subsets of $N \mathcal{M}$
under by the \textit{addition map} $E(p, v) = p + v$,
so we can identify the two without ambiguity.
For example, a \textit{tubular neighborhood} $B$ is such a neighborhood
that is diffeomorphic to a collection of normal disks of continuous radii:
$B = E(\mathscr{D})$, where $\mathscr{D} = \{(p, v) \in N \mathcal{M} : |v| < \delta(p)\}$
and $\delta \in C^0(\mathcal{M}, \mathbb{R}_+)$.
The existence of tubular neighborhoods is guaranteed by
the tubular neighborhood theorem \cite[Thm 6.24]{Lee2012}.
Note that $E$ is bijective on $\mathscr{D}$,
so its restriction $E|_{\mathscr{D}}: \mathscr{D} \mapsto B$ has an inverse:
$\Phi = (E|_{\mathscr{D}})^{-1}$.
A \textit{retraction} from a topological space onto a subspace
is a surjective continuous map that restricts to the identity map on the codomain.
A \textit{smooth submersion} is a smooth map whose differentials are surjective everywhere.
On a tubular neighborhood, we can define a retraction that is also a smooth submersion as such:
$r = \pi_1 \circ \Phi$, $r: B \mapsto \mathcal{M}$.
It is identical to the projection onto $\mathcal{M}$,
that is, $r = P_{\mathcal{M}}|_B$, where $P_{\mathcal{M}}(x) = \arg\min_{p \in \mathcal{M}} \| p - x \|$.
Thus, we will call $r$ the \textit{canonical projection} of $B$,
and denote it as $\pi$, which should not be confused with $\pi_1$.

Fiber bundle is a way to decompose a manifold into a manifold-indexed collection of
homeomorphic manifolds of a lower dimension.
Besides the normal bundle $N \mathcal{M}$, we have now obtained another fiber bundle $(B, \pi, \Phi)$
over $\mathcal{M}$, where $B$ is the total space, $\pi$ is the canonical projection,
$\Phi$ is the trivialization, and $\mathcal{M}$ is the base space.
The \textit{fiber} $\mathcal{F}_p$ over a point $p \in \mathcal{M}$
is the preimage $\mathcal{F}_p = \pi^{-1}(p)$, $\mathcal{F}_p \subset B$. %
In the case of tubular neighborhoods, the fibers are open disks.
For simplicity, we will denote a fiber bundle by its total space, e.g. denote $(B, \pi, \Phi)$ as $B$.
Since $\mathscr{D} = \Phi(B) \subset N \mathcal{M}$, when there is no ambiguity,
we will call $B$ a normal bundle of $\mathcal{M}$.
The normal bundle $(B, \pi, \Phi)$ decomposes the neighborhood $B$ into
a collection $\{\mathcal{F}_p\}_{p \in \mathcal{M}}$ of fibers indexed by the submanifold,
so that every point in the neighborhood can be written uniquely
as the sum of a point on the submanifold and a normal vector.
In the special case of an $\varepsilon$-tubular neighborhood $B_\varepsilon$,
this is a direct sum decomposition: $B_\varepsilon = \mathcal{M} \oplus \mathcal{F}$,
where model fiber $\mathcal{F}$ is an open disk of radius $\varepsilon$ and dimension $c$.

\subsection{Measure and density}
\label{sub:measure}
Probability measures and probability density functions can also be extended to Riemannian manifolds.
A \textit{measure} $\mu$ is a non-negative function on a sigma-algebra of an underlying set $X$,
which is distributive with countable union of mutually disjoint sets.
A natural choice of sigma-algebra for a topological space $(X, \mathcal{T})$
is its \textit{Borel sigma algebra}, %
the sigma-algebra generated by its topology $\mathcal{T}$; this applies to all manifolds.
A \textit{probability measure} is just a normalized measure: $\mu(X) = 1$.
We use superscript to indicate the underlying set of a measure if it is not $\mathbb{R}^n$.
For example, $\mu^\mathcal{M}$ denotes a probability measure on $\mathcal{M}$.

The \textit{Riemannian density} $d V_g$ on $(\mathcal{M}, g)$
is a density uniquely determined by $g$. %
This \textit{density} is not a probability density function,
but a concept defined for smooth manifolds;
the notation $d V_g$ is intended to resemble a volume element.
If $\mathcal{M}$ is compact,
its \textit{volume} $\text{Vol}(\mathcal{M})$ is the integral of its Riemannian density:
$\text{Vol}(\mathcal{M}) = \int_\mathcal{M} d V_g$; and its \textit{Hausdorff measure} $\mathcal{H}^d$
is the integral of its Riemannian density over measurable sets:
$\mathcal{H}^d(A) = \int_A d V_g$, $A \subset \mathcal{M}$.
We obtain a probability measure on $\mathcal{M}$ by normalizing its Hausdorff measure:
$\mu_0^{\mathcal{M}} = \mathcal{H}^d / \text{Vol}(\mathcal{M})$.
Any function $f \in C^0(\mathcal{M}, \mathbb{R}_{\ge 0})$, $\int_\mathcal{M} f d V_g = 1$,
is a probability density function with respect to $\mu_0^{\mathcal{M}}$, in the sense that
it defines a probability measure $\mu^\mathcal{M} = f \mu_0^\mathcal{M}$.
We denote such a probability density function as $p^{\mathcal{M}}$.
Note that $\mu_0^{\mathcal{M}}$ is used here as a reference probability measure,
which can be considered as the uniform distribution on $\mathcal{M}$;
in fact, it is the uniform distribution in the usual sense
if $\mathcal{M}$ has a positive Lebesgue measure.

On a normal bundle $(B, \pi, \Phi)$ over $\mathcal{M}$, any probability measure $\mu^B$
induces a probability measure $\mu^{\mathcal{M}}_{\perp}$ on $\mathcal{M}$ by marginalization:
$\mu^\mathcal{M}_{\perp}(U) = \mu^B(\pi^{-1}(U))$, $U \subset \mathcal{M}$.
Moreover, if $\mu^B$ can be written as $\mu^B = p \mu_0^B$, %
it induces a probability measure $\mu^{\mathcal{F}}$ on each fiber $\mathcal{F}$ by conditioning:
$\mu^{\mathcal{F}} = p^{\mathcal{F}} \mu_0^{\mathcal{F}}$,
where $p^{\mathcal{F}} = p (\int_{\mathcal{F}} p~d V_g)^{-1} \big{|}_{\mathcal{F}}$.

\subsection{Dynamical system}
\label{sub:dynamical}

A \textit{continuous-time dynamical system}, or a \textit{flow}, $\phi: \mathbb{R} \times X \mapsto X$
is a continuous action of the real group $\mathbb{R}$ on a topological space $X$:
$\forall t, t' \in \mathbb{R}$, $\forall x \in X$,
$\phi(0, x) = x$ and $\phi(t', \phi(t, x)) = \phi(t + t', x)$.
If the action is only on the semi-group $\mathbb{R}_{\ge 0}$, we call it a \textit{semi-flow}.
The \textit{trajectory} $\phi_x$ through a point $x \in X$
is the parameterized curve $\phi_x: \mathbb{R} \to X$, $\phi_x(t) = \phi(t, x)$.
The \textit{time-$t$ map} $\phi^t$, $t \in \mathbb{R}$,
is the map $\phi^t: X \mapsto X$, $\phi^t(x) = \phi(t, x)$.
The \textit{time-$\infty$ map} $\phi^\infty$ is the map $\phi^\infty: S \mapsto S$,
$\phi^\infty(x) = \lim_{t \to \infty} \phi^t(x)$, and $S \subset X$ is where the limit exists.
A \textit{vector field} $v(x)$ on a smooth manifold
is a continuous map that takes each point to a tangent vector at that point.
A \textit{flow generated by a vector field}, if exists,
is a differentiable flow such that $\forall x \in X$, $\frac{\partial \phi}{\partial t}(0, x) = v(x)$.

\begin{proposition}[flow]
  \label{prop:flow}
  Let the \textit{subspace-constrained gradient field} $v: D \mapsto \mathbb{R}^n$, $v(x) = L(x) g(x)$.
  If $p(x)$ has bounded super-level sets $B_c = \{x \in \mathbb{R}^n : p(x) \ge c\}$ for all $c > 0$,
  then $v(x)$ generates a semi-flow $\phi: \mathbb{R}_{\ge 0} \times D \mapsto D$.
  If $p(x)$ has a compact support $\overline{D}$, let $v(x) = 0$, $\forall x \in \partial D$,
  then $v(x)$ generates a flow $\phi: \mathbb{R} \times \overline{D} \mapsto \overline{D}$.
  Moreover, if $v(x)$ is locally Lipschitz or $C^k$, $k \ge 1$,
  then $\phi$ is locally Lipschitz or $C^k$, respectively.
\end{proposition}

\begin{proof}%
  Because $p \in C^2(\mathbb{R}^n, \mathbb{R}_{\ge 0})$,
  we have $H = \nabla \nabla p \in C^0(\mathbb{R}^n, \mathcal{S}(n))$,
  where $\mathcal{S}(n) = \{A \in \mathbb{R}^{n \times n} : A = A^{\text{T}}\}$.
  So the subspace $\text{Span}(V_c)$ spanned by the eigenvectors of the bottom-$c$ eigenvalues of $H$
  is continuously varying: $\text{Span}(V_c) \in C^0(\mathbb{R}^{n \times n}, G_{c, n})$,
  where the Grassmann manifold $G_{c, n}$ consists of $c$-subspaces of $\mathbb{R}^n$.
  This means the projection matrix $L = V_c V_c^{\text{T}}$ is also continuously varying.
  Since $g = \nabla p \in C^1(\mathbb{R}^n, \mathbb{R}^n)$, we have $v(x) = L(x) g(x)$ is continuous,
  and therefore it is a vector field on $\mathbb{R}^n$.
  Let $\partial B_c$ be the boundary of $B_c$.
  For each $x \in \partial B_c$, if $g(x) \ne 0$,
  by the regular level set theorem \cite[Thm 3.2]{Hirsch1976},
  there is a neighborhood $U(x) \subset \mathbb{R}^n$ such that
  $\partial B_c \cap U(x)$ is a $C^2$ hypersurface in $\mathbb{R}^n$.
  Additionally, $g(x) \in N_x \partial B_c$ points in the inward normal direction.
  Therefore, the projection of $g(x)$ onto any subspace would still points inwards or vanish,
  which applies to $v(x) = L(x) g(x)$.
  If $g(x) = 0$, apparently $v(x) = 0$.
  So for all $x \in \partial B_c$, $v(x)$ points into $B_c$ or vanish.
  Because $B_c$ is a closed set and assumed to be bounded, it is compact.
  Thus, the vector field $v(x)$ is forward complete, that is, it generates a unique semi-flow
  $\phi: \mathbb{R}_{\ge 0} \times B_c \mapsto B_c$.
  As $c \to 0$, $B_c$ expands to $\mathbb{R}^n$, so $v(x)$ generates a semi-flow on $\mathbb{R}^n$.
  If $p(x)$ is compactly supported, then so is $v(x)$,
  therefore $v(x)$ is complete and generates a unique global flow
  $\phi: \mathbb{R} \times \overline{D} \mapsto \overline{D}$.
\end{proof}

\begin{proposition}[convergence]
  \label{prop:convergence}
  If $p(x)$ is analytic and has bounded super-level sets,
  then every forward trajectory converges to a fixed point:
  $\forall x \in \mathbb{R}^n$, $\exists x^* \in v^{-1}(0)$,
  $\lim_{t \to +\infty} \phi_x(t) = x^*$.
\end{proposition}

\begin{proof}%
  When $v(x) \ne 0$, we have $v = L g = V_c V_c^{\text{T}} g \ne 0$,
  which means $V_c^{\text{T}} g \ne 0$ and therefore $(v, g) = g V_c V_c^{\text{T}} g > 0$.
  Because $(v, g) = g V_c V_c^{\text{T}} g \ge 0$, we have $(v, g) = 0$ implies $v = 0$.
  Let $\theta(u, w) = (u, w) / (\|u\| \|w\|)$,
  then $\{x : \theta(g, v) = \pi / 2\} \subset v^{-1}(0)$.
  Let $\delta \in [0, \pi / 2]$ and
  $U(\delta) = \{x : \theta(g, v) \ge \pi / 2 - \delta\} \subset v^{-1}(0)$.
  Let $\tilde{v}(x) = 0$ if $x \in U(\delta)$ and $\tilde{v}(x) = v(x)$ otherwise.
  Let $\tilde{\phi}$ be the semi-flow generated by $\tilde{v}$.
  Then $\forall \delta > 0$, $\forall x \in D \setminus U(\delta)$, $\forall t \in \mathbb{R}_{\ge 0}$:
  $\theta(g(\xi), \tilde{v}(\xi)) \le \pi / 2 - \delta$, where $\xi = \tilde{\phi}_x(t)$.
  By Lojasiewicz’s theorem with an angle condition (see \cite{Lageman2002,Absil2005}),
  either $\lim_{t \to +\infty} \|\tilde{\phi}_x(t)\| = \infty$
  or $\exists x^* \in \mathbb{R}^n$, $\lim_{t \to +\infty} \tilde{\phi}_x(t) = x^*$.
  Because $p(\tilde{\phi}_x(t))$ is non-decreasing and $p(x)$ has compact super-level sets,
  $\tilde{\phi}_x(t)$ must converge to a point $x^*$.
  And because $\tilde{v}(x^*) = 0$, we have $x^* \in U(\delta)$.
  Let $x^\dagger = \lim_{\delta \to 0+} x^*$, then $x^\dagger \in U(0) = v^{-1}(0)$.
\end{proof}

Due to the convergence property of $\phi$, we can focus on its fixed points.
For $\phi$, the set of asymptotically stable fixed points is $\mathcal{R}$,
which can be easily checked by definition.
In fact, $\mathcal{R}$ is the attractor of $\phi$, %
and nearby trajectories approach along normal directions \cite[Lemma 8]{Genovese2014}.
The \textit{basin of attraction} $\mathcal{B}$ of $\mathcal{R}$
is the union of images of all trajectories that tend towards it:
$\mathcal{B} = \{x \in \mathbb{R}^n : \phi^\infty(x) \in \mathcal{R}\}$.
By \cite[Lemma 3]{Genovese2014},
$\mathcal{B}$ contains an $\varepsilon$-tubular neighborhood $B_\varepsilon$ of $\mathcal{R}$,
where $\mathcal{R}$ is exponentially attractive.
Here we show that, %
under a stronger manifold assumption, %
$\mathcal{B}$ is a set of probability one.

\begin{proposition}[basin]
  \label{prop:basin}
  If $p(x)$ is analytic and has a compact support $\overline{D}$,
  and $A_u = \{x \in D : v = 0, \lambda_c > 0\}$ and $A_c = \{x \in D : v = 0, \lambda_c = 0\}$
  are, respectively, embedded $d$- and $(d-1)$-submanifolds of $\mathbb{R}^n$,
  then $\mathcal{B}$ is a subset of full Lebesgue measure and therefore has probability one:
  $\lambda(D \setminus \mathcal{B}) = 0$, $\mu(\mathcal{B}) = 1$,
  where $\lambda$ is the Lebesgue measure on $\mathbb{R}^n$.
\end{proposition}

\begin{proof}%
  $v^{-1}(0)  = \mathcal{R} \sqcup A_u \sqcup A_c$,
  and $\forall x \in v^{-1}(0)$, $W(x) = \{\xi \in D : \lim_{t \to +\infty}\phi_\xi(t) = x\}$
  has dimension at most $c$.
  For $x \in A_u$, $x$ has an unstable manifold of dimension at least one,
  so $D_u = \{x \in D : \lim_{t \to +\infty}\phi_x(t) \in A_u\}$ has Lebesgue measure zero.
  Because $A_c$ has dimension $d - 1$, and $d - 1 + c = n - 1 < n$,
  so $D_c = \{x \in D : \lim_{t \to +\infty}\phi_x(t) \in A_c\}$ also has Lebesgue measure zero.
  Because $D = \mathcal{B} \sqcup D_u \sqcup D_c$,
  so $\lambda(D \setminus \mathcal{B}) = \lambda(D_u \sqcup D_c) = \lambda(D_u) + \lambda(D_c) = 0$.
  Because $\mu(D) = 1$ and $\mu(D \setminus \mathcal{B}) \le \max(p)
  \lambda(D \setminus \mathcal{B}) = 0$,
  so $\mu(\mathcal{B}) \ge \mu(D) - \mu(D \setminus \mathcal{B}) = 1$
  and therefore $\mu(\mathcal{B}) = 1$.
\end{proof}

Now we have (yet another) fiber bundle $(\mathcal{B}, \pi, \Phi)$ over $\mathcal{R}$,
where canonical projection $\pi(x) = \phi^\infty(x)$
and trivialization $\Phi(x) = (\pi(x), x - \pi(x))$.
$\pi$ is a retraction that approximates the projection
$P_{\mathcal{R}}(x) = \arg\min_{p \in \mathcal{R}} \| p - x \|$
to the second order \cite[Lem 2.8]{ZhangRD2020nr}.
If $\phi$ is smooth within $\mathcal{B}$, then $\pi$ is a smooth submersion.
Because the fiber $\mathcal{F}_r = \pi^{-1}(r)$ over each point $r \in \mathcal{R}$
is a level set of $\pi$, by the submersion level set theorem \cite[Cor 5.13]{Lee2012},
it is a properly embedded $c$-submanifold.
In terms of the dynamical system, each $\mathcal{F}_r$ is a stable manifold,
because every forward trajectory starting on $\mathcal{F}_r$ stays within $\mathcal{F}_r$
and converges to $r$.
Because $\mathcal{F}_r$ intersects $\mathcal{R}$ at $r$ orthogonally,
we will also call $(\mathcal{B}, \pi, \Phi)$ a normal bundle over $\mathcal{R}$,
when there is no ambiguity.
See \Cref{fig:overview}d for commutative diagrams of this bundle
and its restriction to a subset of the ridge.
As with the general case discussed earlier,
any probability measure $\mu^{\mathcal{B}}$ on $\mathcal{B}$
induces a probability measure $\mu^{\mathcal{R}}_{\perp}$ on $\mathcal{R}$ by marginalization:
$\mu^{\mathcal{R}}_{\perp} = \mu^{\mathcal{B}} \circ \pi^{-1}$.
But the dynamical system offers a more explicit perspective on this marginalization process:
$\phi$ continuously transforms $\mu^{\mathcal{B}}$ towards $\mathcal{R}$ such that
at $t > 0$, $\mu_t^{\mathcal{B}} = \mu^{\mathcal{B}} \circ (\phi^t)^{-1}$,
and the induced measure is the asymptotic measure,
$\mu^{\mathcal{R}}_{\perp} = \lim_{t \to \infty} \mu_t^{\mathcal{B}}$.

\section{Algorithm}
\label{sec:alg}

We formally describe normal-bundle bootstrap in \cref{alg:nbb}, and analyze its properties.
Let $K_h$ be a density kernel with bandwidth $h$,
density estimate $\hat{p}_{h}(x) = N^{-1} \sum_{i=1}^N K_h(x - x_i)$, and
leave-one-out density estimate $\hat{p}_{h, -i}(x) = (N-1)^{-1} \sum_{j \ne i} K_h(x - x_j)$.
Let $\alpha \in (1, +\infty)$ be an oversmoothing factor and
$k \in \{0, \dots, N\}$ be the number of nearest neighbors.
A standing assumption of the algorithm is that $d$ is small,
so that $N$ does not have to be too large for good estimation.
On the other hand, $n$ can be reasonably large under typical computational constraints.

\alglanguage{pseudocode}
\begin{algorithm}[h]
  \caption{NormalBundleBootstrap$(X, d, \alpha, k)$}
  \label{alg:nbb}
  \begin{algorithmic}[1] %
    \State $h \gets \alpha \arg\max_h \sum_{i=1}^N \log \hat{p}_{h, -i}(x_i)$
    \Comment{kernel bandwidth selection} %
    \State $(\hat{r}_i, V_{c,i}) \gets \text{SCMS}(x_i; \log \hat{p}_{h}, d)$, for $i \in N$
    \Comment{ridge estimation} %
    \State $E \gets \text{SmoothFrame}(\hat{r}, V_c, n - d)$
    \Comment{align bottom-$c$ eigenvectors} \label{line:frame}
    \State $[\hat{n}]_i \gets E_i^{\text{T}} (x_i - \hat{r}_i)$, for $i \in N$
    \Comment{coordinates of normal vectors} \label{line:coordinate} %
    \State $K \gets \text{KNN}(\hat{r}, k)$
    \Comment{$k$-nearest neighbors on ridge}
    \State $\tilde{x}_{ij} \gets \hat{r}_i + E_i [\hat{n}]_{K(i, j)}$, for $i \in N, j \in k$
    \Comment{construct new data} \label{line:construct}
  \end{algorithmic}
\end{algorithm}

In this algorithm, smooth frame construction (line \ref{line:frame}) and
coordinate representation (line \ref{line:coordinate}) can be removed to save computation,
but with less desirable results.
In this case data construction (line \ref{line:construct}) directly uses
projection vectors $x_l - \hat{r}_l$ or normal vectors $\hat{n}_{il} = L_i (x_l - \hat{r}_l)$,
where $l = K(i,j)$.
Algorithms \texttt{SCMS} and \texttt{SmoothFrame} are given in supplementary materials.

\subsection{Qualitative properties of the dynamical system}
\label{sub:qualitative}

In \cref{sub:dynamical} we have shown that, under suitable conditions,
subspace-constrained gradient field $v$ generates a flow $\phi$ whose attractor is $\mathcal{R}$,
and the basin of attraction $\mathcal{B}$ is a fiber bundle
with canonical projection $\pi = \phi^\infty$.
The dynamical system $\phi$ is determined by $p$.
Because the dynamical system is stable \cite[Thm 4]{Genovese2014},
$p$ can be replaced by an estimate $\hat{p}$
to obtain an attractor $\hat{\mathcal{R}} = \text{Ridge}(\hat{p}, d)$ that approximates $\mathcal{R}$.
Here we use a density estimate $\hat{p}_h$ with Gaussian kernel $K_h(x) \propto \exp(-x^2/(2h^2))$.
Denote the generated flow as $\phi_N$ and
the estimated ridge as $\hat{\mathcal{R}}_N = \text{Ridge}(\hat{p}_h, d)$.

It is preferable to define $\phi_N$ by $\log \hat{p}_h$ instead of $\hat{p}_h$.
Note that $\text{Ridge}(\log \hat{p}_h, d) = \hat{\mathcal{R}}_N$.
If $\phi_N$ is defined by $\log \hat{p}_h$, then $\mathcal{B}$ is larger %
and independent of the size of normal space distribution \cite[Thm 7]{Genovese2014},
and trajectories are more orthogonal to $\hat{\mathcal{R}}_N$ (see \Cref{fig:dynamical}).
Moreover, $\hat{\mathcal{R}}_N$ is exponentially stable within $\mathcal{B}$,
as $v$ is approximately linear in normal spaces %
\cite[Lemma 8]{Genovese2014}.

The attractor $\hat{\mathcal{R}}_N$ may be bounded or unbounded.
If $\mathcal{R}$ is a compact submanifold without boundary, as is often assumed in previous studies,
$\hat{\mathcal{R}}_N$ can be compact and without boundary.
If $\mathcal{R}$ has a boundary, $\hat{\mathcal{R}}_N$ would be unbounded, see \Cref{fig:dynamical}.
This is also true if $\mathcal{R}$ is noncompact, as is the Gaussian example in \Cref{fig:overview}a.
In such cases, although finite data is always bounded, the attractor will be unbounded.

\begin{figure}[t]
  \centering
  \includegraphics[width=\linewidth]{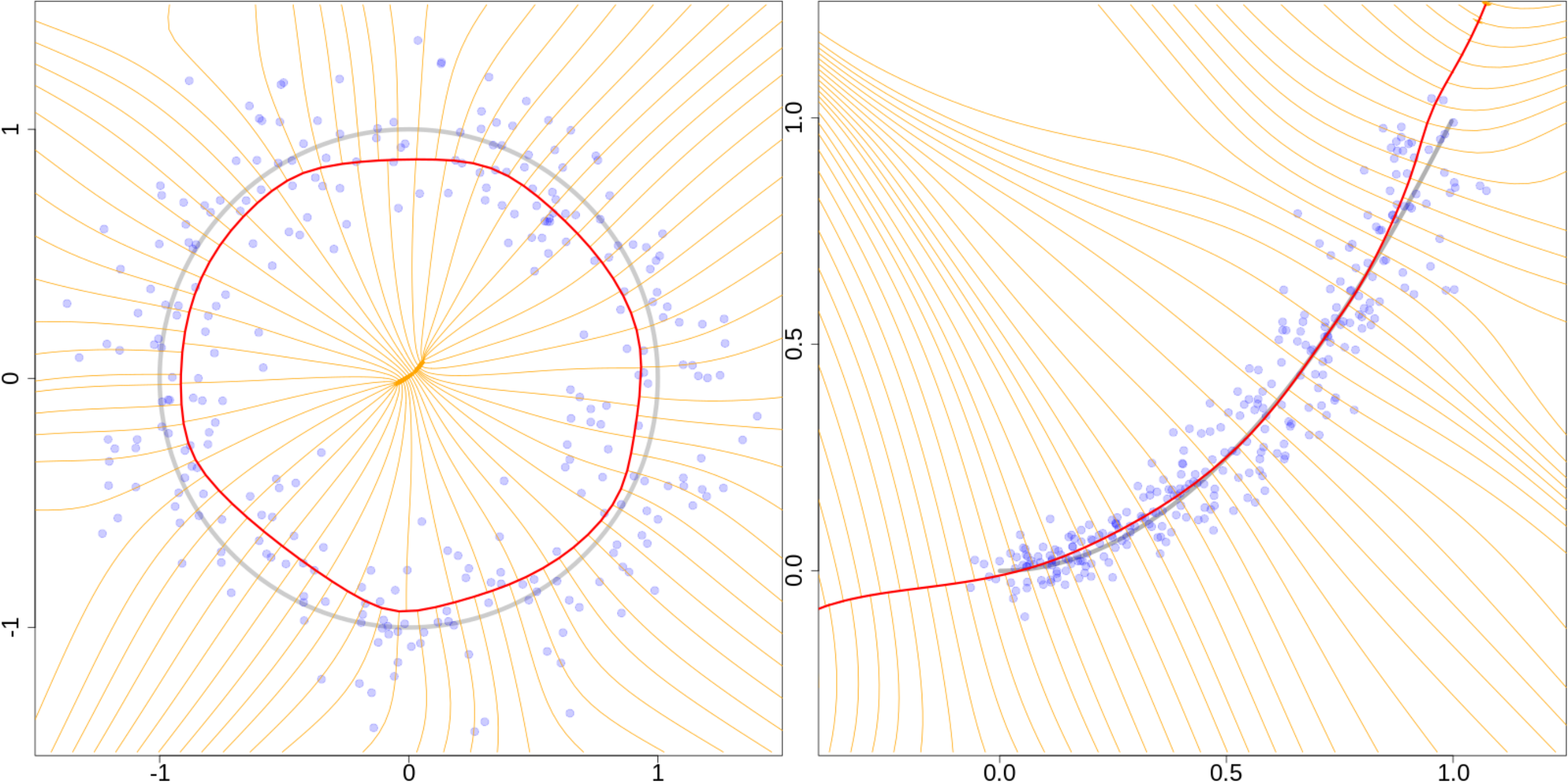}
  \caption{Subspace-constrained gradient flow as projection to estimated density ridge.
    Data (blue points); true (gray curve) and estimated (red curve) density ridge;
    trajectories (orange curves), pointing towards estimated ridge.
    (a) True ridge is the unit circle, a manifold without boundary;
    the estimated ridge is also without boundary.
    (b) True ridge is a parabola segment, a manifold with boundary;
    the estimated ridge is unbounded.}
  \label{fig:dynamical}
\end{figure}

\subsection{Statistical properties}
\label{sub:statistical}

As in \cref{sub:measure,sub:dynamical},
the normal bundle $(\mathcal{B}, \pi, \Phi)$ over the density ridge $\mathcal{R}$
decomposes the original probability measure $\mu$
into a ``marginalized measure'' $\mu^{\mathcal{R}}_{\perp}$ on the ridge and
a ``conditional measure'' $\mu^{\mathcal{F}_r}$ on each fiber, where $r \in \mathcal{R}$.
If we know $\mu^{\mathcal{R}}_{\perp}$ and each $\mu^{\mathcal{F}_r}$, we can sample $\mu$ as follows:
first sample $\mathbf{r} \sim \mu^{\mathcal{R}}_{\perp}$, and then sample $\mu^{\mathcal{F}_{\mathbf{r}}}$.
Although such measures are unknown, we can still estimate them from available data,
and use them for inference and data augmentation.
Here we show that normal-bundle bootstrap constructs new data points
that are consistent with the conditional measures on the normal spaces,
and have nice finite-sample validity.

\begin{assumption}[\cite{Genovese2014}, Sec 2.2]
  \label{ass:advanced}
  In a neighborhood $B$ of ridge $\mathcal{R}$, %
  (A0) $p(x)$ is three times differentiable;
  (A1) $p(x)$ is sharply curved in normal spaces:
  $\lambda_c < - \beta$ and $\lambda_c < \lambda_{c+1} - \beta$, where $\beta > 0$; %
  (A2) trajectories $\phi_x(t)$ are not too wiggly and
  tangential gradients $U(x) g(x)$ are not too large:
  $\|U(x)g(x)\| \max_{i,j,k}\left|\frac{\partial H_{ij}}{\partial x_k}(x)\right| <
  \frac{\beta^2}{2n^{3/2}}$.
\end{assumption}

\begin{theorem}[consistency]
  \label{thm:consistency}
  Let \cref{ass:advanced} hold for the measure $\mu$ in the basin of attraction $\mathcal{B}$,
  and the conditional measure $\mu^{\mathcal{F}_r}$ varies slowly over the ridge $\mathcal{R}$,
  then for each estimated ridge point
  $\hat{\mathbf{r}} = \pi_N(\mathbf{x}) = \phi_N^\infty(\mathbf{x})$,
  as sample size $N \to \infty$,
  the distributions of the constructed data points $\tilde{\mathbf{x}}_j$, $j \le k$,
  converge to the distribution restricted to the fiber of the estimated ridge point:
  $\tilde{\mathbf{x}}_j | \hat{\mathbf{r}} \xrightarrow{d} \mathbf{x}|_{\mathcal{F}_{\hat{r}}}$.
\end{theorem}

\begin{proof}
The normal bundle $(\mathcal{B}, \pi_N)$ over the estimated ridge $\hat{\mathcal{R}}_N$
decomposes the original measure $\mu$ into the marginalized measure
$\mu^{\hat{\mathcal{R}}_N}_{\perp}$ and the conditional measures $\mu^{\mathcal{F}_{\hat{r}}}$,
$\hat{r} \in \hat{\mathcal{R}}_N$.
Because the data is distributed as the original measure, $\mathbf{x} \sim \mu$,
each estimated ridge point is then distributed as the marginal measure,
and the normal vector at each estimated ridge point is distributed as
the conditional measure at that ridge point:
$\hat{\mathbf{r}} \sim \mu^{\hat{\mathcal{R}}_N}_{\perp}$ and
$\mathbf{n} | \hat{\mathbf{r}} \sim \mu^{\mathcal{F}_{\hat{r}}}$.

Since the Gaussian kernel is smooth, the density estimate $\hat{p}_h(x)$ satisfies condition (A0).
By \cite[Thm 5]{Genovese2014}, as sample size $N$ goes to infinity,
the estimated ridge $\hat{\mathcal{R}}_N$ within the basin of attraction $\mathcal{B}$
converges to the true ridge: $\lim_{N \to \infty} \text{Haus}(\mathcal{R}, \hat{\mathcal{R}}_N) = 0$,
where the Hausdorff distance between two sets is defined as
$\text{Haus}(A, B) = \max\{\sup_{x \in A} d(x, B), \sup_{x \in B} d(x, A)\}$.
Because the conditional measures $\mu^{\mathcal{F}_r}$ over the true ridge $\mathcal{R}$ vary slowly,
and the estimated ridge approximates the true ridge,
the conditional measures $\mu^{\mathcal{F}_{\hat{r}}}$ over the estimated ridge $\hat{\mathcal{R}}_N$
also vary slowly.
For an estimated ridge point $\hat{\mathbf{r}}$, the normal vectors at its $k$-nearest neighbors
$\hat{\mathbf{r}}_j$, $j \le k$, are thus distributed similarly to the normal vector at this point:
$\mathbf{n}_j | \hat{\mathbf{r}}_j \sim \mu^{\mathcal{F}_{\hat{r}_j}}$,
$\mu^{\mathcal{F}_{\hat{r}_j}} \approx \mu^{\mathcal{F}_{\hat{r}}}$.
As sample size $N$ goes to infinity, the distances to its $k$-nearest neighbors vanishes:
$\lim_{N \to \infty} d(\hat{\mathbf{r}}, \hat{\mathbf{r}}_j) = 0$.
Therefore, the distributions of neighboring normal vectors converge to the distribution
of the normal vector at the estimated ridge point:
$\lim_{N \to \infty} \mu^{\mathcal{F}_{\hat{r}_j}} = \mu^{\mathcal{F}_{\hat{r}}}$.
Note that this limit is understood in the sense of a metric on measure spaces,
such as the Wasserstein metrics.
The constructed data points add neighboring normal vectors to the estimated ridge point,
$\tilde{\mathbf{x}}_j | \hat{\mathbf{r}} = \hat{\mathbf{r}} + \mathbf{n}_j$;
as a result, their distributions converge to the original measure
restricted to the fiber of the estimated ridge point:
$\tilde{\mathbf{x}}_j | \hat{\mathbf{r}} \xrightarrow{d}
\hat{\mathbf{r}} + \mathbf{n} | \hat{\mathbf{r}} \sim \mathbf{x}|_{\mathcal{F}_{\hat{r}}}$.
\end{proof}

We have shown that the normal-bundle bootstrapped data
have desirable large-sample asymptotic behavior,
but their finite-sample behavior is also very good.
In fact, as soon as the estimated ridge becomes close enough to the true ridge
such that the conditional measures $\mu^{\mathcal{F}_{\hat{r}}}$ over the estimated ridge vary slowly,
the conditional measures on neighboring fibers become similar to each other:
$\mu^{\mathcal{F}_{\hat{r}_j}} \approx \mu^{\mathcal{F}_{\hat{r}}}$.
This would suffice to make the constructed data distribute similarly
to the original measure restricted to a fiber:
$\tilde{\mathbf{x}}_j | \hat{\mathbf{r}}~\dot\sim~\mathbf{x}|_{\mathcal{F}_{\hat{r}}}$.
Even if the estimated ridge has a finite bias to the true ridge, see e.g. \Cref{fig:dynamical}a,
it would not affect the conclusion.
Suppose the true ridge is the unit circle and the conditional measures $\mu^{\mathcal{F}_r}$
over the true ridge are identical, if the estimated ridge is a circle of a smaller radius,
then the conditional measures $\mu^{\mathcal{F}_{\hat{r}}}$ over the estimated ridge are also identical,
but with a constant bias to $\mu^{\mathcal{F}_r}$.
Despite such a bias, the constructed data will have the same distribution as the restricted measure:
$\tilde{\mathbf{x}}_j | \hat{\mathbf{r}} \sim \mathbf{x}|_{\mathcal{F}_{\hat{r}}}$.
We will illustrate the finite-sample advantage of NBB in \cref{sec:experiments}.

\subsection{Computational properties}
\label{sub:computational}

SCMS \cite{Ozertem2011} is an iterative algorithm that updates point locations by
$x_{t+1} = x_t + s(x_t)$, where $s(x) = L(x) m(x)$ is the subspace-constrained mean-shift vector
and $m(x)$ is the mean-shift vector.
If density estimate $\hat{p}_h$ uses a Gaussian kernel with bandwidth $h$,
then $m(x) = h^2 \hat{g}_h(x) / \hat{p}_h(x)$,
where $\hat{g}_h(x) = \nabla \hat{p}_h(x)$ is the plug-in estimate of density gradient.
A naive implementation of SCMS would have a computational complexity of $O(N^2 n^3)$ per iteration,
where the $O(N^2)$ part comes from computing for each update point $x_t$ using all data points,
and the $O(n^3)$ part comes from eigen-decomposition of the Hessian.
Although estimates of density, gradient, and Hessian all need to be computed for each update point,
the most costly operation is the eigen-decomposition.

However, a better implementation can reduce the computational complexity to
$O(k d n^2)$ per iteration for one update point.
Here we use the $k$-nearest data points,
assuming that the more distant points have negligible contribution to the estimated terms.
And we use partial eigen-decomposition to obtain the top $d$ eigen-pairs in $O(d n^2)$ time.

Another direction to accelerate computation is by reducing the number of iterations.
Recall that the attractor $\hat{\mathcal{R}}_N$ is exponentially stable,
therefore $\{x_t\}_{t \in \mathbb{N}}$ is linearly convergent.
We can use Newton's method for root finding to achieve quadratic convergence.
For $x$ in a neighborhood $B$ of $\hat{\mathcal{R}}_N$,
let subspace $S = \text{Span}(V_c)$, affine space $A = x + S$,
and let $C$ be the component of $A \cap B$ containing $x$.
Then ridge point $r = C \cap \hat{\mathcal{R}}_N$ is the unique zero of $v|_C$ and it is regular.
Recall that $v = L g$, $L = V_c V_c^{\text{T}}$, Newton's method for $v|_C = 0$ updates by
$x_{t+1} = x_t + L_t \delta_t$, where $\delta_t$ solves
$L_0 H_t L_0 \delta_t = - L_0 g_t$ or $L_t H_t L_t \delta_t = - L_t g_t$.
Both converge quadratically near $\hat{\mathcal{R}}_N$,
while the former only requires (partial) eigen-decomposition at the first step,
and the latter has a larger convergence region \cite[Lem 2.12]{ZhangRD2020nr}.

\section{Experiments}
\label{sec:experiments}

In this section we showcase the application of normal-bundle bootstrap
in inference and data augmentation, using two simple examples.

\subsection{Inference: confidence set of density ridge}
\label{sub:inference}

Normal-bundle bootstrap constructs new data points
that approximate the distributions on normal spaces of the estimated density ridge,
and thus can be used for inference of population parameters of these distributions.
For example, it can provide confidence sets of
the true density ridge via repeated mode estimation in each normal space,
and provide confidence sets of principal manifolds \cite{Hastie1989}
via repeated mean estimation in each normal space.

For a confidence set $\hat{C}_N$ of $\mathcal{R}$,
it is asymptotically valid as a uniform confidence set at level $1 - \alpha$ if
$\liminf_{N \to \infty} P(\mathcal{R} \subset \hat{C}_N) \ge 1 - \alpha$;
similarly, it is valid as a pointwise confidence set if
$\liminf_{N \to \infty} \mathbb{E}\mu^\mathcal{R}_0(\mathcal{R} \subset \hat{C}_N) \ge 1 - \alpha$.
Pointwise confidence sets are less conservative and can be more useful.
We define an NBB pointwise confidence set $\hat{C}^{\text{NBB}}_N = \hat{\mathcal{R}}_N \oplus D_\alpha =
\{\hat{r} + \hat{n} : \hat{r} \in \hat{\mathcal{R}}_N, \hat{n} \in D_\alpha(\hat{r})\}$,
where disk $D_\alpha(\hat{r}) = \hat{m} \oplus \varepsilon_\alpha =
\{\hat{n} \in N_{\hat{r}} \hat{\mathcal{R}}_N : d(\hat{n}, \hat{m}) < \varepsilon_\alpha \}$.
For $\hat{r}_i$, $\hat{m}_i$ is the mode estimated from the constructed points $\tilde{x}_{ij}$.
Radius $\varepsilon_{\alpha}$ is determined by $P(d(m, \hat{m}) < \varepsilon_{\alpha}) = 1 - \alpha$,
where $m$ is the mode of $p|_{\mathcal{F}_{\hat{r}}}$ and corresponds to
$\mathcal{R} \cap \mathcal{F}_{\hat{r}}$; %
its estimator $\hat{\varepsilon}_{\alpha}$
is the $\alpha$-upper quantile of $\{d(\hat{m}^*_b, \hat{m})\}_{b = 1}^B$,
where $\hat{m}^*$ denotes a bootstrap estimate using a bootstrap resample of the constructed points.
Note that an NBB pointwise confidence set for a principal manifold can be defined
simply by replacing $m$ and $\hat{m}$ with mean and sample mean.

Alternatively, confidence sets for $\mathcal{R}$ can also be obtained by bootstrap.
\cite{ChenYC2015} showed that a bootstrap uniform confidence set $\hat{C}^{\text{B}}_N$
converges in Hausdorff distance at a rate of $O(N^{-1/2})$
to the smoothed density ridge $\mathcal{R}_h = \text{Ridge}(p_h, d)$,
where smoothed density $p_h = p * K_h$ and $*$ denotes convolution.
Here, $\hat{C}^{\text{B}}_N = \hat{\mathcal{R}}_h \oplus \varepsilon_\alpha =
\{x \in \mathbb{R}^n : d(x, \hat{\mathcal{R}}_h) < \varepsilon_\alpha \}$
is the $\varepsilon_\alpha$-uniform tubular neighborhood of $\hat{\mathcal{R}}_h$,
the estimated ridge using kernel bandwidth $h$.
Radius $\varepsilon_\alpha$ is determined by
$P(d_\Pi (\hat{\mathcal{R}}_h, \mathcal{R}_h) < \varepsilon_\alpha) = 1 - \alpha$, where
$d_\Pi (\hat{\mathcal{R}}_h, \mathcal{R}_h) = \sup_{x \in \mathcal{R}_h} d(x, \hat{\mathcal{R}}_h)$;
its estimator $\hat{\varepsilon}_\alpha$ is the $\alpha$-upper quantile of
$\{d_\Pi(\hat{\mathcal{R}}^*_b, \hat{\mathcal{R}}_h)\}_{b = 1}^B$.
A bootstrap pointwise confidence set of $\mathcal{R}_h$ can be similarly defined
where $\varepsilon_\alpha$ is determined by
$P(d(\mathbf{r}, \hat{\mathcal{R}}_h) < \varepsilon_\alpha) = 1 - \alpha$
and estimator $\hat{\varepsilon}_\alpha$ is the $\alpha$-upper quantile of
$\{d(\hat{r}^*_{i,b}, \hat{\mathcal{R}}_h)\}^{b = 1 \dots B}_{i = 1 \dots N}$.
But if $N$ is small and therefore $h$ is large,
$\mathcal{R}_h$ can have large bias from $\mathcal{R}$,
so the bootstrap confidence sets can have poor coverage of $\mathcal{R}$.

Here we compare the pointwise confidence sets of density ridge by NBB and bootstrap.
As an experiment, data are sampled uniformly on the unit circle,
and a Gaussian noise is added in the radial direction:
$\mathbf{x} = \mathbf{r} e^{i \boldsymbol{\theta}}$,
$\boldsymbol{\theta} \sim U[0, 2\pi)$, $\mathbf{r} \sim N(1, 0.2^2)$.
The 1d density ridge of $\mathbf{x}$ is numerically identical with the unit circle.
\Cref{fig:circle}(a-b) illustrates $\hat{C}^{\text{NBB}}$ and $\hat{C}^{\text{B}}$
on a random sample,
and \Cref{fig:circle}(c-d) compares their finite-sample validity and average compute time
over independent samples.
$\hat{C}^{\text{NBB}}$ is valid throughout the range of sample sizes computed,
while the validity of $\hat{C}^{\text{B}}$ slowly improves.
Moreover, $\hat{C}^{\text{B}}$ is computationally costlier than $\hat{C}^{\text{NBB}}$,
due to repeated ridge estimation.
Although repeated mode estimation is also costly,
it is faster than ridge estimation of the same problem size,
and the constructed points in each normal space is only a fraction of the original sample.
Specifically, the computational complexity of $\hat{C}^{\text{B}}$ is $O(n^3 N^2 B)$ per iteration,
from bootstrap repetitions of ridge estimation;
that of $\hat{C}^{\text{NBB}}$ is $O(n k N B)$,
where $O(n k)$ comes from estimating gradient using $k$ constructed data points,
and $O(N B)$ comes from computing for all normal spaces and all bootstrap repetitions.
Note that other population parameters like mean and quantiles
can be estimated much faster than the mode,
so the related inference using NBB will be much faster than in this example,
such as confidence sets of principal manifolds.

\begin{figure}[t]
  \centering
  \includegraphics[width=.9\linewidth]{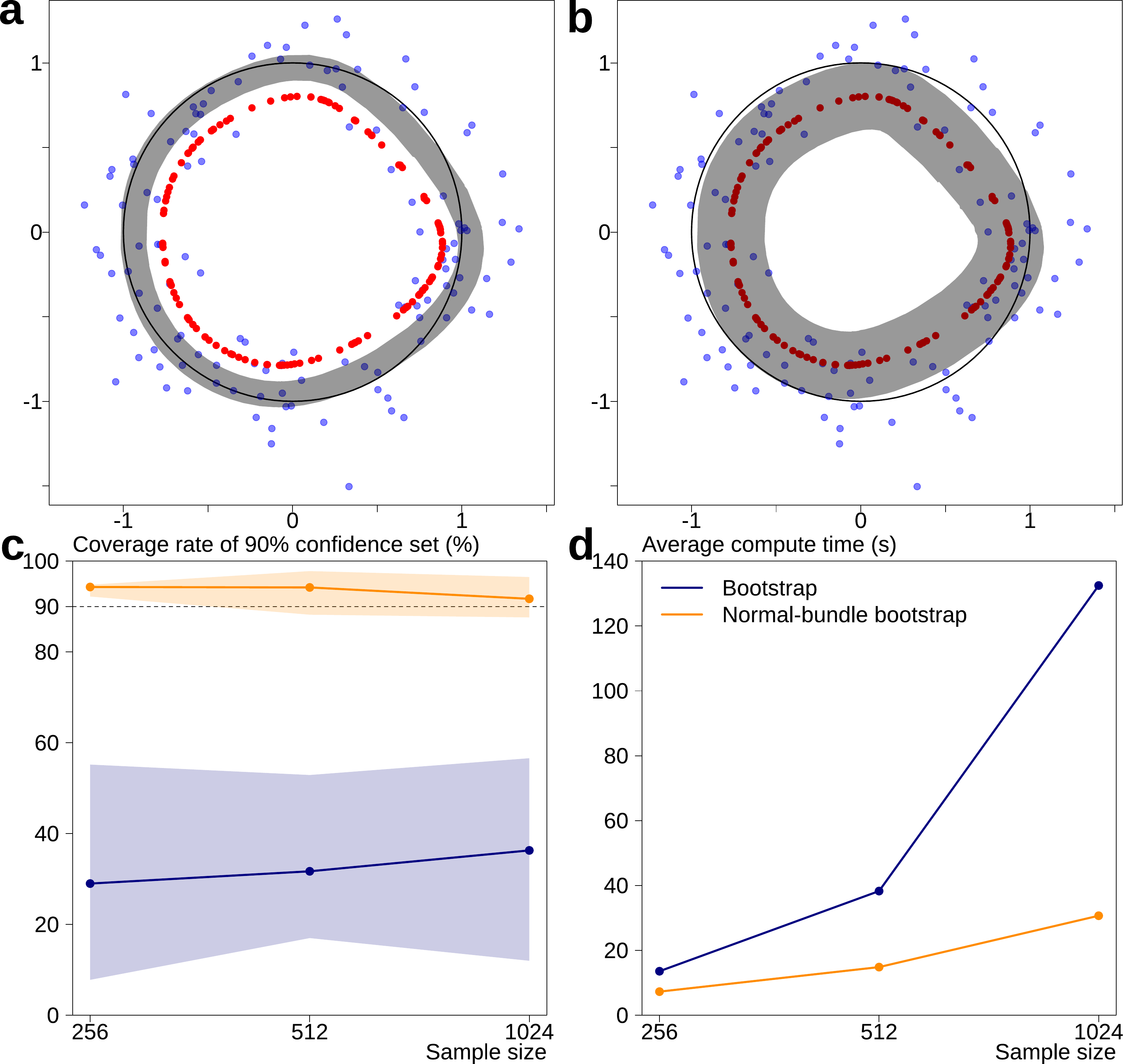}
  \caption{Inference.
    (a-b) 90\% confidence sets of density ridge:
    data (blue), estimated ridge (red), true ridge (black),
    confidence sets (gray) by NBB (a) vs. bootstrap (b). N = 128.
    (c-d) metrics of NBB (orange) and bootstrap (blue) over an ensemble of samples:
    (c) coverage rate, mean (solid line) and 90\% prediction interval (shade);
    (d) average computation time.}
  \label{fig:circle}
\end{figure}

\subsection{Data augmentation: regression by deep neural network}
\label{sub:augmentation}

For machine learning tasks, the data constructed by normal-bundle bootstrap
can be used to augment the original data to avoid overfitting.
The idea behind this is that when the amount of training data
is insufficient for a model not to overfit, but enough for a good estimate of the density ridge,
we can include the NBB constructed data to increase the amount of training data.
Because for each estimated ridge point, the NBB constructed data is balanced around the true ridge
in the sense that their estimated mode is near the true ridge point,
so the augmented training data can resist overfitting to the noises.

Here we consider a regression problem with one input parameter and a functional output.
Let $\mathbb{S}^1 \subset \mathbb{R}^2$ be the unit circle,
$\theta \in \mathbb{R}$ be a rotation angle (with unit $\pi$),
$\tau: \mathbb{S}^1 \mapsto \mathbb{S}^1$ be the map
between initial and final configurations of the circle,
and $f$ be the relationship between $\theta$ and $\tau$ such that $f(\theta) = \tau$.
The task is to learn $f$ from data.
We discretize the circle into a set of $l$ random points with initial angles
$\{\gamma_j \pi\}_{j=1}^l \subset [0, 2 \pi)$.
Under the true model, when $\theta = \theta_i$ their coordinates can be written as
$(x_{ij}, y_{ij}) = (\cos(\pi (\theta_i + \gamma_j)), \sin(\pi (\theta_i + \gamma_j)))$.
Assume that all variables are subject to measurement error such that we can only observe
$\tilde{\boldsymbol\theta} \sim N(\theta, 0.2^2)$ and
$\tilde{\mathbf{x}}_j, \tilde{\mathbf{y}}_j \sim N(\theta, 0.2^2)$, $j \in \{1, \dots, l\}$.
We obtain training data $(\tilde{\theta}_i, (\tilde x_{ij}, \tilde y_{ij})_{j=1}^l)_{i=1}^N$,
and obtain another set of data for validation.
Specifically, we have $l = 8$ and $N = 32$, so the training data is a $32 \times 17$ matrix.

For the neural network, we use a sequential model with four densely connected hidden layers,
which have 256, 128, 64, and 32 units respectively and use the ReLU activation function;
the output layer has 16 units.
We train the network to minimize mean squared error.
For data augmentation, we set $k = 16$ in NBB, and combine the constructed data with training data.
\Cref{fig:wheel} illustrates the original and augmented training data,
and compares the training and validation errors with and without data augmentation.
We can see that without augmentation the network starts to overfit around epoch 100,
while with augmentation the network trains faster, continues to improve over time,
and has a lower error.

\begin{figure}[t]
  \centering
  \includegraphics[width=\linewidth]{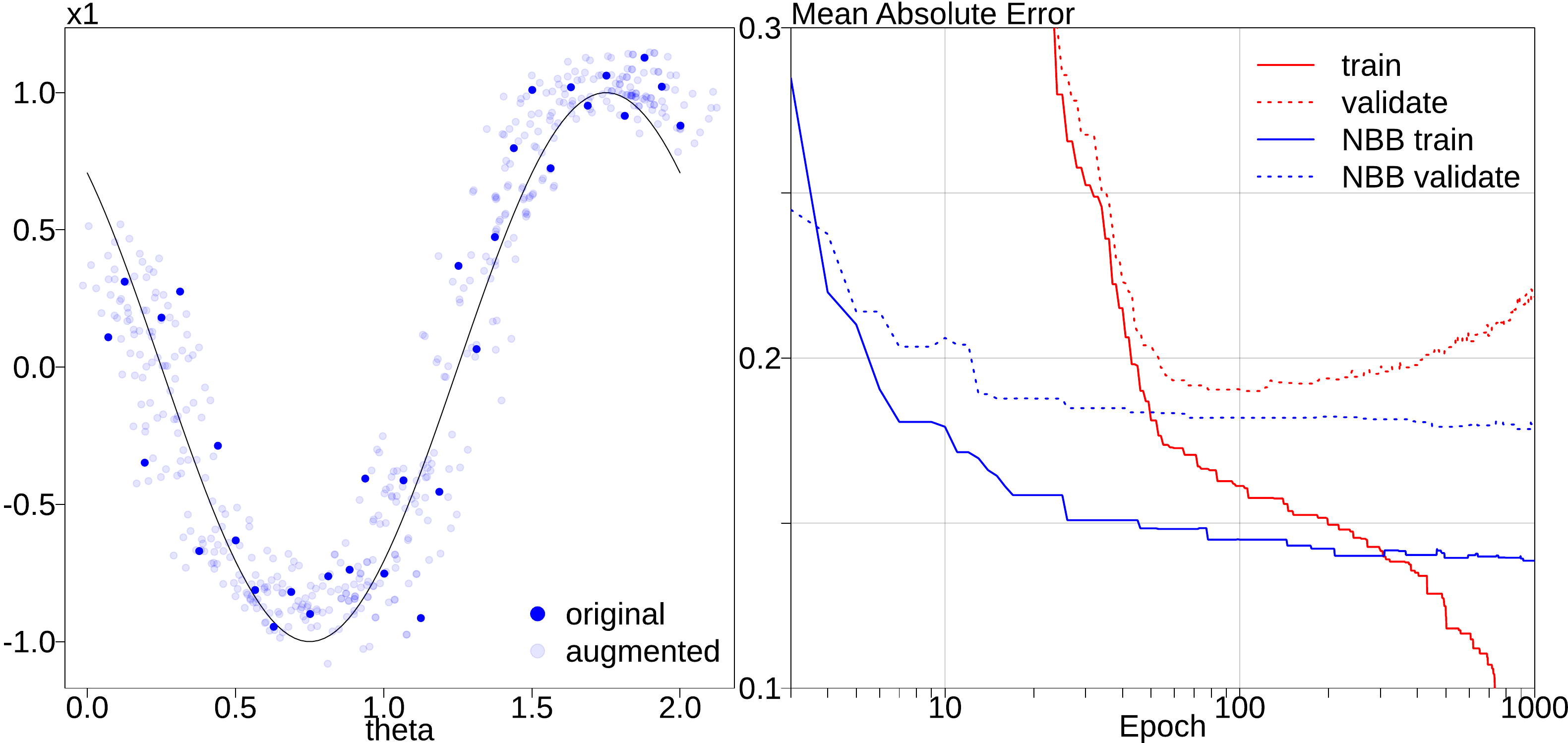}
  \caption{Data augmentation.
    (left) original and augmented data, showing $(\theta, x_1)$ only.
    Noiseless true model in black line.
    (right) training and validation error with and without NBB.}
  \label{fig:wheel}
\end{figure}

\section{Discussion}
\label{sec:discussion}

In this section we discuss the determination of hyper-parameters for NBB:
kernel bandwidth $h$, ridge dimension $d$, and number of neighbors $k$.

Kernel bandwidth $h$ should be selected for optimal estimation of the density ridge.
A good estimate should resemble the shape of the true ridge while bias can be well tolerated,
because with a smooth frame, NBB can correct for bias away from the estimated ridge.
Silverman's rule-of-thumb bandwidth tends to oversmooth the ridge,
because the true density is supposed to have a salient geometric structure
rather than been an isotropic Gaussian.
Maximum likelihood bandwidth tends to be too small,
such that the estimated ridge often has isolated points.
We use an oversmoothing parameter $\alpha$, usually between 2 and 4,
and good estimates can be often obtained across a wide range of $\alpha$ values.
\cite{ChenYC2015b} gave a method to select $h$ that minimizes coverage risk estimates.

Ridge dimension $d$ is often apparent in specific problems.
In low-dimensional problems with $n \le 3$, the structure can often be examined visually.
In regression, $d$ is the number of explanatory variables.
In identifying implicit relations in a system, such as by symbolic regression or sparse regression,
$d = n - c$ is the system's degree of freedom, where $c$ is the number of constraint equations.
If data is generated from a manifold, possible subject to ambient noise,
$d$ is the manifold dimension.
If no external information is available to determine $d$, we can 
use eigengaps of the Hessian $H = \nabla \nabla \log \hat{p}_h$:
find $c \in \{1, \dots, n-1\}$ with the largest $\min\{\lambda_{c+1} - \lambda_c : x \in X\}$.

Number of neighbors $k$ determines the amount of new data constructed by NBB,
and we would prefer it to be as large as possible.
For an estimated ridge point $\hat{r}_i$,
$k_i$ should not exceed the largest local smooth frame containing the point.
And the faster the distributions on normal spaces vary over the ridge, the smaller $k$ should be.
If a global smooth frame can be constructed and the noises are identical across the ridge,
we can set $k = n$.
Typically, we let $k = \varepsilon n$, with $\varepsilon \in (0, 1/2]$.
One criteria is that given $\hat{r}_i$, the normal vectors $[\hat{n}]_{K(i, j)}$ should be uni-modal.
So if mode estimation on $[\hat{n}]_{K(i, j)}$ gives multiple points, $k_i$ should be decreased.

\section{Conclusion}
\label{sec:conclusion}

We introduced normal-bundle bootstrap,
a method to resample data sets with salient geometric structure.
The constructed new data are consistent with the distributions on normal spaces,
and we demonstrated its uses in inference and data augmentation.

\section*{Acknowledgments}
The authors thank Ernest Fokoue of Rochester Institute of Technology for valuable discussions.

\appendix

\section{Algorithms}

Here are some algorithms used in \cref{alg:nbb} for normal-bundle bootstrap.
\texttt{KNN} for $k$-nearest neighbors is a common algorithm and therefore not listed.

\texttt{SCMS} is an implementation of subspace-constrained mean shift \cite{Ozertem2011}
for ridge estimation, where we use the logarithm of a Gaussian kernel density estimate.
Note that density estimate $\hat{p}_h$ in the algorithm input is replaced with $(X, h)$
since we are assuming a Gaussian kernel.
Note that this is naive implementation can be accelerated using local data and Newton-like methods.

\alglanguage{pseudocode}
\begin{algorithm}[h]
  \caption{SCMS$(y; X, h, d, \theta_0 = 0.05)$}
  \label{alg:scms}
  \begin{algorithmic}[1] %
    \Repeat
    \State $z_i \gets (x_i - y) / h$, for $i \in N$
    \State $c_i \gets \exp(-\text{sum}(z_i^2) / 2)$, for $i \in N$
    \State $p_i \gets c_i / \sum_{i=1}^N c_i$, for $i \in N$
    \State $r_{pz,i} \gets \sqrt{p_i} z_i$, for $i \in N$
    \State $s_{pz} \gets \sum_{i=1}^N p_i z_i$
    \State $r_{pz} r_{pz}^{\text{T}} - s_{pz} s_{pz}^{\text{T}} = V \Lambda V$
    \Comment{eigen-decomposition}
    \State $m_c \gets N^{-1} \sum_{i=1}^N c_i$
    \State $m_{cz} \gets N^{-1} \sum_{i=1}^N c_i z_i$
    \State $m \gets h m_{cz} / m_c$ \Comment{mean-shift vector}
    \State $s \gets (I - V_d V_d^{\text{T}}) m$ \Comment{SCMS vector}
    \State $\theta \gets m^T s / \sqrt{\text{sum}(m^2)~\text{sum}(s^2)}$
    \State $y \gets y + s$
    \Until{$\theta > \theta_0$} \Comment{convergence criteria}
    \State \textbf{return} $(y, V, \lambda)$
  \end{algorithmic}
\end{algorithm}

\alglanguage{pseudocode}
\begin{algorithm}[h]
  \caption{SmoothFrame$(R, V_c, c, j = 1)$} %
  \label{alg:sf}
  \begin{algorithmic}[1] %
    \State $(K, D) \gets \text{KNN}(R, N-1)$
    \Comment{index and distance matrices of nearest neighbors}
    \State $k \gets \text{repeat}(1, N)$
    \Comment{$K,D$-indices of nearest unaligned neighbor} %
    \State $E[j] \gets V_c[j]$ \Comment{initial reference orthonormal $c$-frame}
    \State $b \gets 1$ \Comment{$a$-index of the last aligned point}
    \While{$b < N$}
        \State $a[b] \gets j$ \Comment{indices of points in order of alignment}
        \State $\text{replace}(K, j, \text{NULL})$ \Comment{remove indices of aligned points}
        \ForAll{$i$ in $a$} \Comment{maintain the property of $k$}
            \While{$K[i, k[i]]$ is NULL}
                \State $k[i] \gets k[i] + 1$
            \EndWhile
        \EndFor
        
        \State $i \gets a[\text{which.min}(D[a, k[a]])]$
        \Comment{index of aligned point closest to the unaligned}
        \State $j \gets K[i, k[i]]$ \Comment{index of the next point to align}
        \State $\text{Align}(j, i)$ \Comment{align $E[j]$ to $E[i]$}
        \State $b \gets b + 1$ %
    \EndWhile
    \State \textbf{return} $E$
    \Statex
    \Procedure{Align}{$j, i$} %
      \State $E[j] \gets V_c[j]$ \Comment{initial orthonormal $c$-frame}
      \State $\Theta \gets E[j]^{\text{T}} E[i]$ \Comment{cosine matrix to reference frame $E[i]$}
      \State $\Theta = A \Sigma B^{\text{T}}$ \Comment{singular value decomposition}
      \State $Q \gets A B^{\text{T}}$ \Comment{rotation matrix}
      \State $E[j] \gets E[j] Q$ \Comment{aligned orthonormal $c$-frame}
    \EndProcedure
\end{algorithmic}
\end{algorithm}

\texttt{SmoothFrame} constructs smooth frames of the normal bundle of an estimated density ridge,
where procedure \texttt{Align} adapts the moving frame algorithm \cite{Rheinboldt1988}
for the normal bundle.
This algorithm recursively aligns the nearest unaligned point,
which is ``optimized'' for stability but not for speed.
It might be faster if using one reference frame for a neighborhood,
such that the neighborhoods cover the data set.
Moreover, when $c$ is large, only the top among the bottom-$c$ eigenvectors
are significant to correct for biases introduced in ridge estimation,
so a smooth subframe of the normal bundle suffice, which saves computation and storage.
For the remaining normal directions, assuming negligible bias to the true ridge
and radial symmetry (in addition to unimodality) of noise distribution,
one may bootstrap the norm of the residual noise and multiply it with a random residual direction.

This algorithm, as written, assumes that a smooth global frame exists
for the normal bundle of the estimated density ridge, or equivalently,
that the the normal bundle is trivial.
The normal bundle of a density ridge does not need to be trivial, or not even orientable.
Consider the uniform distribution on a Mobius band in the Euclidean 3-space,
under a small additive Gaussian noise, the 2d density ridge includes the band,
so the estimated density ridge approximates the band, which is non-orientable.
Therefore, an (estimated) density ridge does not need to admit
a smooth global frame for its normal bundle.
In case the normal bundle is not trivial,
several smooth frames need to be constructed to cover the ridge.
In terms of computation, one needs to run this algorithm on several subsets of the estimated ridge,
such that for every point on the ridge,
there is a frame that contains enough neighbors to the point.

On the other hand, for a constraint manifold, i.e. regular level set $\mathcal{M} = F^{-1}(0)$,
its normal bundle is trivial (see \cite[10-18]{Lee2012}),
admits a smooth global frame (see \cite[10.20]{Lee2012}),
and it is orientable (see \cite[15-8]{Lee2012}); in particular,
the Jacobian $J^{\text{T}}(x)$ is a smooth/$C^{k-1}$ global frame for $N \mathcal{M}$.
By QR decomposition where $R$ has all positive diagonal entries,
$Q(x)$ is smooth/$C^{k-1}$ orthonormal global frame for $N \mathcal{M}$.
Because non-orientable submanifolds of Euclidean spaces (e.g. the Mobius band)
do not have global frames, they cannot be constraint manifolds.

\section{List of Symbols} 
Here we provide the system of symbols we used in this article.

Manifold:

\begin{itemize}
\item $\mathbb{R}^n$, Euclidean n-space;
\item $(\mathcal{M}, g)$, Riemannian submanifold of dimension $d$ with induced Riemannian metric;
\item $\mu_0^\mathcal{M}(A) = \int_A d V_g / \int_{\mathcal{M}} d V_g$, normalized Hausdorff measure,
a reference probability measure on the submanifold;
\item $p^\mathcal{M}$, $\mu^\mathcal{M} = p^\mathcal{M} \mu_0^\mathcal{M}$, probability density/measure
on the submanifold;
\item $T_p \mathcal{M}$, $N_p \mathcal{M}$, tangent/normal space at a point on the submanifold;
\item $N \mathcal{M} = \sqcup_{p\in\mathcal{M}} N_p \mathcal{M}$, normal bundle of the submanifold;
\end{itemize}

Fiber bundle:

\begin{itemize}
\item $(B, \pi, \Phi)$, fiber bundle, a tuple of total space, projection, and trivialization;
\item $\mathcal{M} = \pi(B)$, base space of the bundle, a manifold;
\item $\mathcal{F}_r = \pi^{-1}(r)$, fiber over a point on the base space;
\item $B|_S = \pi^{-1}(S)$, restriction of a fiber bundle to a subset of its base space;
\item $\Phi(x) = (\pi(x), x - \pi(x))$, trivialization of the normal bundle;
\item $\mathscr{D} = \Phi(B) \subset N \mathcal{M}$, trivialized normal bundle;
\item $\mu^\mathcal{M}_{\perp} = \mu \circ \pi^{-1}$, measure induced by projection on the base space;
\item $\mu^\mathcal{F} = p^\mathcal{F} \mu_0^\mathcal{F}$, $p^\mathcal{F} = \frac{p}{\int_\mathcal{F} p~d V_g} \bigg{|}_\mathcal{F}$,
  measure induced on each fiber, and its density function;
\end{itemize}

Dynamical system:

\begin{itemize}
\item $g(x) = \nabla p(x)$, $H(x) = \nabla \nabla p(x)$, gradient/Hessian of density function;
\item $V$, $\Lambda = \text{diag}(\lambda)$, matrices of eigenvectors/eigenvalues of the Hessian;
\item $U = V_d V_d^{\text{T}}$, $L = I - U$, orthonormal frames of the top-$d$/bottom-$c$
eigenvectors of the Hessian;
\item $\mathcal{R} = \{x \in \mathbb{R}^n : \lambda_c(x) < 0, L(x) g(x) = 0\}$, density ridge
of dimension $d$;
\item $v(x) = L(x) g(x)$, subspace-constrained gradient field;
\item $\phi(t, x)$, $\phi^t(x)$, semi-flow generated by $v$, and its time-$t$ map;
\item $(\mathcal{B}, \phi^\infty)$, normal bundle of the density ridge
(basin of attraction as total space, and time-infinite map as projection);
\item $U$, neighborhood on density ridge;
\end{itemize}

Algorithm:

\begin{itemize}
\item $X$, data set of $N$ points;
\item $\hat{p}_h(x)$, estimated density function with kernel bandwidth $h$;
\item $m(x) = h^2 \hat{g}_h(x) / \hat{p}_h(x)$, mean-shift vector based on Gaussian kernel;
\item $s(x) = L(x) m(x)$, subspace-constrained mean-shift vector;
\item $\alpha$, smoothing factor;
\item $k$, number of nearest neighbors;
\item $\mathbf{x}_i$, $\hat{\mathbf{r}}_i = \pi(\mathbf{x}_i)$,
$\hat{\mathbf{n}}_i = \mathbf{x}_i - \hat{\mathbf{r}}_i$,
$\tilde{\mathbf{x}}_{ij} = \hat{\mathbf{r}}_i + \hat{\mathbf{n}}_{i_j}$, data point, ridge point,
normal vector, and constructed data point;
\end{itemize}

\bibliographystyle{siamplain}
\bibliography{nbb.bib}

\end{document}